\pdfoutput=1

\documentclass[11pt]{article}

\usepackage[]{acl}

\usepackage{times}
\usepackage{latexsym}

\usepackage[T1]{fontenc}

\usepackage[utf8]{inputenc}

\usepackage{microtype}

\usepackage[normalem]{ulem}
\usepackage{amsmath}
\usepackage{amssymb}
\usepackage{graphicx}
\usepackage{amsthm}
\newtheorem{definition}{Definition}
\newtheorem{thm}{Theorem}

\usepackage{booktabs}

\title{Differentially Private Decoding in Large Language Models}

\author{Jimit Majmudar \; Christophe Dupuy \; Charith Peris \; Sami Smaili \\ {\bf Rahul Gupta} \; {\bf Richard Zemel} \\
Amazon Alexa AI  \\
\texttt{\{mjimit,dupuychr,perisc,smaili,gupra,rzemel\}@amazon.com} }

\begin{document}
\maketitle
\begin{abstract}
Recent large-scale natural language processing (NLP) systems use a pre-trained Large Language Model (LLM) on massive and diverse corpora as a headstart. In practice, the pre-trained model is adapted to a wide array of tasks via fine-tuning on task-specific datasets. LLMs, while effective, have been shown to memorize instances of training data thereby potentially revealing private information processed during pre-training. The potential leakage might further propagate to the downstream tasks for which LLMs are fine-tuned. On the other hand, privacy-preserving algorithms usually involve retraining from scratch, which is prohibitively expensive for LLMs. In this work, we propose a simple, easy to interpret, and computationally lightweight perturbation mechanism to be applied to an already trained model at the decoding stage. Our perturbation mechanism is model-agnostic and can be used in conjunction with any LLM. We provide theoretical analysis showing that the proposed mechanism is differentially private, and experimental results showing a privacy-utility trade-off.
\end{abstract}

\section{Introduction}
Language Models (LMs) are trained to generate (a sequence of) tokens by learning to predict the next or missing tokens in a sentence. In practice, their output is a probability distribution over the vocabulary words from which the predicted token is drawn. Contemporary LMs are parametrized as Neural Networks (NNs) comprising millions or even billions of parameters; we refer to such LMs as Large Language Models (LLMs).
LLMs, trained on large corpora, are now widely available and are at the core of many language-based commercial services such as question-answering systems, conversational artificial intelligence, and sentence completion/correction.
Initially, LSTM-based sequence-to-sequence (seq2seq) models \citep{sutskever2014sequence} were proposed as LLMs, but in the recent past, transformer-based models \citep{vaswani2017attention, devlin2018bert, Radford2019LanguageMA, Conneau2020UnsupervisedCR} have superseded them in performance. 

Despite the performance gain, recent studies have shown that LLMs tend to unintentionally memorize portions of their training data \citep{Zhang_counter_2021, Carlini2019TheSS, Carlini2022QuantifyingMA}. Such memorization raises the risk of these models regurgitating parts of their training data, verbatim, in response to an appropriately-crafted prompt. This is problematic in cases where these models have been trained on private customer data.
In addition, intense memorization also has the potential to degrade utility and hurt fairness by not generalizing well to the true data distribution \cite{generalization2017Kawaguchi,brown2021memorization,Feldman2020memorization}. 

The extent of memorization in LLMs correlates with multiple factors. Explorations on generative models, such as GPT-2 \citep{Radford2019LanguageMA}, have shown that the main factors are the scale of the model, repetitions in the train data, and how much context was used in the extraction process \citep{Carlini2022QuantifyingMA}. \citet{Carlini2022QuantifyingMA} also suggest that most published attacks only establish a lower bound on the memorization level, with their example showing that at least 1\% of GPT-J \citep{gpt-j} training dataset is memorized. With no upper bound on memorization, it is imperative that we explore ways to mitigate any privacy risks posed by using these ``off-the-shelf'' models. 

At a first glance, we may attribute the problem of memorization as that of over-fitting to the training data. \citet{Carlini2019TheSS} invalidate this hypothesis and show that various regularization techniques such as early stopping and dropout are insufficient to mitigate memorization. The authors argue that model memorization is prevented only when using Differential Privacy (DP) training. DP, first introduced in the seminal work by \citet{dwork2006calibrating}, is a property of an algorithm defined as follows. 

\begin{definition}[$\epsilon$-Differential Privacy for Training]
Let $\mathcal{X}$ and $\mathcal{Y}$ be the sets of features and labels respectively. A randomized training algorithm $\mathcal{A}$ is $\epsilon$-differentially private if for any two datasets $S, S' \in (\mathcal{X} \times \mathcal{Y})^n$ differing in exactly one entry, and for all sets $W$ in the range of $\mathcal{A}$
\begin{equation*}
    \ln\left(\frac{\Pr[\mathcal{A}(S) \in W]}{\Pr[\mathcal{A}(S') \in W]} \right) \leq \epsilon.
\end{equation*}
\end{definition}
This definition intuitively suggests that the output of algorithm $\mathcal{A}$ does not vary too much whether a certain single data point is in the training set or not. 

In general, $\epsilon$ is seen as a privacy loss such that a smaller value corresponds to better privacy. 
The definition of DP is broader relative to the notion of memorization and therefore DP algorithms are usually robust to other threats such as data linkage attacks, side channel attacks, and membership inference attacks \citep{Jagannatha2021MembershipIA, Yeom2017TheUC}. 

While the intersection of DP and LLMs is fairly novel, the prominent approach to introduce DP into machine learning models has been to modify the training process (e.g., Differentially Private Stochastic Gradient Descent (DP-SGD) \cite{song2013stochastic,Abadi2016DPSGD}). 
Intuitively, these modifications ensure that the trained model does not depend too much on certain training data points thereby guaranteeing privacy; albeit the privacy gains come at significant costs in computation, training time, and model performance. Recent works \citep{li2021large, yu2021differentially, dupuy2021efficient} address the aforementioned challenges, but DP training still remains an active area of interest as no conclusive method satisfactorily balances all aspects. In a real-world commercial setting, the cost considerations are even magnified: retraining production-scale LLMs from scratch, without a noticeable loss in performance, would entail a massive time and monetary expense and is therefore an unappealing option.

More recently, there has been interest in introducing DP to already trained machine learning models at the prediction stage \citep{dwork2018privacy, dagan2020pac} as formalized in the following definition.

\begin{definition}[$\epsilon$-Differential Privacy for Prediction]
Let $\mathcal{X}$ and $\mathcal{Y}$ be the sets of features and labels respectively. A prediction algorithm $f$, whose weights are determined by a training algorithm $\mathcal{A}$, is $\epsilon$-differentially private if for any two datasets $S, S' \in (\mathcal{X} \times \mathcal{Y})^n$ differing in exactly one entry, for all $x \in \mathcal{X}$, and for all sets $Y \subseteq \mathcal{Y}$
\begin{equation*}
    \ln\left(\frac{\Pr[f(x;\mathcal{A}(S)) \in Y]}{\Pr[f(x; \mathcal{A}(S')) \in Y]} \right) \leq \epsilon.
\end{equation*}
\end{definition}

Intuitively, this definition suggests that the prediction of a model, on a possibly adversarially selected input feature, does not vary too much whether or not a certain single data point is in the training set. Note that notion of DP prediction is weaker than that of DP training in the sense that the latter implies the former, due to sequential composition property \citep{mironov2017renyi}. To our best knowledge, prior to this work, only \citet{ginart2022submix} study DP prediction for LLMs; the authors propose an ensemble-based method which significantly increases computational and storage costs.

Our contributions: 1) We propose a simple and computationally lightweight method for DP prediction in LLMs based on perturbation to the output probability distribution and show that our approach is differentially private; 2) We apply our method to a pre-trained LLM and compare the impact on performance and privacy.

\section{DP Decoding in LLM}

\subsection{Problem Formulation}
\label{sec:problem}
Suppose we have a vocabulary of tokens ${V = \{w_1, \dots, w_{|V|} \}}$. Denote a trained LM by $\mathcal{M}:V^*\to V^T$,
which takes as input an arbitrarily long sequence of tokens and a set of positions for which tokens need to be predicted, and outputs a probability distribution over $V$ for each token to be predicted. We assume that the model outputs at most $T$ tokens. If $\mathcal{M}$ is a Masked Language Model (MLM), then we input a sentence with some masked tokens, and $\mathcal{M}$ outputs a probability distribution for each of the masked tokens. Similarly, if $\mathcal{M}$ is a sequence-to-sequence model, then we input a sequence of tokens from $V$, and $\mathcal{M}$ outputs a sequence of probability distributions over $V$.

There are different decoding or sampling strategies which transform the probability distribution(s) output by the model to actual token(s).
Exploring such strategies is still a problem of interest to the community \citep{meister2022typical, holtzman2019curious}. It is this part of the prediction procedure that we modify, and in this work, we focus on random sampling, where the $i^{th}$ predicted token $t_i$ is drawn from the vocabulary according to the discrete distribution $q_i$ output by the model: $t_i|C_i\sim q_i\in\Delta^{|V|-1}$, with $C_i$ being the context for predicting $t_i$ (e.g., the sentence in which $t_i$ is predicted) and $\Delta^{|V|-1}$ being the probability simplex of dimension $|V|$.
Our proposed method can also be applied to other diverse sampling functions, for which DP analysis remains a future exercise. 

In DP prediction, we use a black box threat model which posits that an adversary can freely query the model and observe its outputs (they do not have access to the weights or the architecture of the trained model). By definition of the DP prediction setup, the model is not necessarily trained with DP, so the raw model does not have any privacy guarantee.
This is unlike the DP training paradigm in which the resulting model has a fixed privacy loss. The model has to be retrained from scratch every time a change in privacy loss is needed (e.g., for tuning privacy and utility), which is computationally expensive especially for LLMs.
Our approach is more practical as it does not require any retraining and the privacy guarantees can be adjusted on the go.
We consider the following question:
\begin{center}
\emph{How can we make LM predictions differentially private under a black box threat model?}    
\end{center}

\subsection{Perturbation Mechanism}
We propose a simple perturbation mechanism applied to each probability distribution $q\in\Delta^{|V|-1}$ output by model $\mathcal{M}$. We evaluate the privacy improvement due to our mechanism as a trade-off with model utility, measured via perplexity. Our perturbation mechanism is a lightweight post-processing step at the decoding stage, and therefore introduces negligible computational overhead. 

The intuition behind our perturbation mechanism is as follows. A probability distribution $q$ output by the model is the best in terms of utility of the model with respect to some underlying true data distribution, but it is worst in terms of privacy since it makes the prediction based on the data seen during training. On the other hand, the uniform distribution $u$ over the vocabulary is best in terms of privacy since it does not reveal anything about the training data, but is the worst in terms of utility for the same reason. In our approach, we obtain a perturbed distribution $q'$ by linear interpolation between the original distribution $q$ and the uniform distribution $u$:
\begin{equation}
    \label{eq:prob-pert}
    q' = \lambda q + (1-\lambda)u.
\end{equation}
We use the perturbed probability $q'$ to randomly sample a token from the vocabulary as described in Section~\ref{sec:problem}, i.e., for the $i^{th}$ predicted token $t_i$  we have ${\forall k\in\{1, \ldots, |V|\}}$, ${\Pr[t_i=w_k \mid C_i]=q'_i(k)}$, where $q'_i(k)$ denotes the $k^{th}$ entry of $q'_i$ and $C_i$ the context used to predict $t_i$.
The parameter $\lambda$ can naturally be interpreted as a privacy-utility trade-off parameter: $\lambda = 0$ yields the best privacy / worst utility while $\lambda = 1$ yields the worst privacy / best utility. In Figure~\ref{fig:simplex} we illustrate this perturbation with an example.

\begin{figure}
    \centering
    \includegraphics[width=\linewidth]{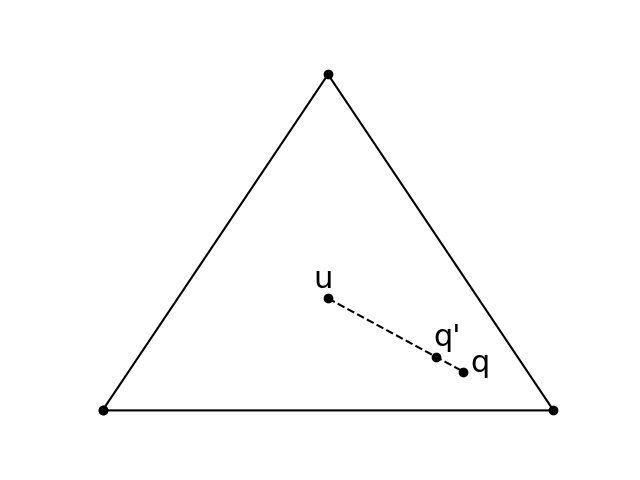}
    \caption{Probability simplex over a vocabulary with $|V|=3$. The probability output by the model $q$ is perturbed to $q'$ which lies on the line segment joining $q$ and the uniform distribution vector $u$.}
    \label{fig:simplex}
\end{figure}

The following result provides a theoretical justification for the privacy benefit of the proposed perturbation mechanism. 

\begin{thm}\label{thm:eps-dp}
Suppose an LM $\mathcal{M}$ uses the perturbation mechanism in Equation (\ref{eq:prob-pert}), for some $\lambda \in [0, 1)$, followed by random sampling. Then $\mathcal{M}$ is an $\epsilon$-differentially private prediction algorithm, with 
\begin{equation}\label{eq:dp_guarantees}
\epsilon=T \log \left( \dfrac{1+ (|V|-1)\lambda}{1- \lambda} \right).    
\end{equation}
\end{thm}
\begin{proof}
See Appendix~\ref{ap:theory}.
\end{proof}
\section{Experimental Results}
In this section, we show the privacy-utility trade-off resulting from the application of the proposed perturbation mechanism. We use a RoBERTa-style encoder model specified by \citet{shoeybi2019megatron} trained with an MLM objective. The model is parametrized by around 2 billion parameters and is trained on a corpora of 1 trillion tokens comprising Common Crawl\footnote{\url{https://commoncrawl.org/}}, Wikipedia\footnote{\url{www.wikipedia.org}} and mC4 \cite{mt5_2021} datasets; the size of the vocabulary is around $|V|=150,000$ tokens. We select $125,000$ examples from the training set. We randomly mask 15\% of the tokens (note that 15\% is standard practice in literature \cite{devlin2018bert}) and pass the masked sentences through the model.

For showing the privacy-utility trade-off, we represent the privacy loss as the $\epsilon$ derived in Theorem~\ref{thm:eps-dp}
(a smaller $\epsilon$ corresponds to better privacy guarantee) and the utility of the model is captured as the corpus-level perplexity. 
Observe that $T$ occurs in our theoretical analysis as an upper bound on the actual number of tokens predicted for an example.
In practice, since we mask (and predict) 15\% of tokens for each sentence, the number of predicted tokens (hence the privacy loss $\epsilon$) varies from a sentence to the other. We therefore report the average $\epsilon$ over the corpus which is computed by setting $T$ as the average number of masked tokens per example.


For each value of $\lambda\in\{0.1, 0.2, \ldots, 0.9, 1\}$ we compute the perplexity for the corresponding perturbed output distributions. For the chosen set of $125,000$ examples we repeat the masking and evaluation steps 3 times and average the perplexity over these 3 restarts. Note that for $\lambda=0$, we have the best privacy guarantees $\epsilon=0$ but the worst perplexity value $\mathrm{PPL}=|V|=150000$, since $\forall i, k, \Pr[t_i=w_k \mid C_i]=1/|V|$.


\begin{figure}
    \centering
    \includegraphics[width=\linewidth]{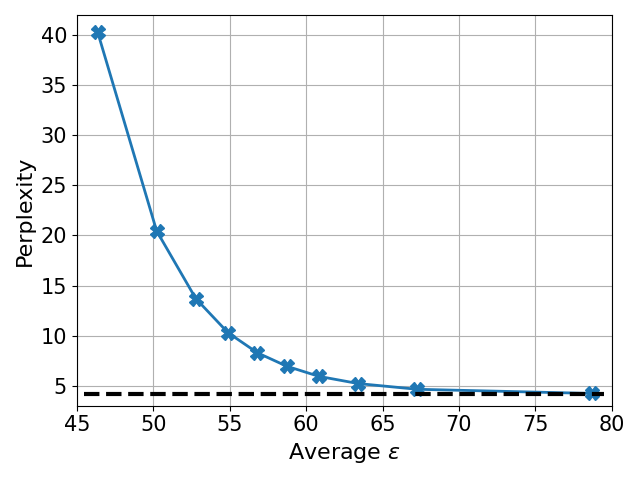}
    \caption{Perplexity as a function of DP-$\epsilon$ values. The dotted line is the perplexity of the model without privacy guarantee.}
    \label{fig:pu}
\end{figure}

\subsection{Discussion}
Privacy/utility results are presented in Figure \ref{fig:pu}. Additionally, we show examples of model predictions at different privacy levels in Appendix~\ref{ap:examples}.
As expected, the model performance degrades as the privacy guarantee improves. In particular, we observe that the perplexity of the model is significantly degraded for $\epsilon\leq 60$ (i.e. $\lambda \leq 0.7$ with perplexity more than $40\%$ higher than for the non-private model). The model becomes unusable at $\epsilon\leq 50$ (i.e. $\lambda \leq 0.2$) where the perplexity is above 20 ($\approx350\%$ relative increase compare to the non-private model). For higher values of $\epsilon$ ($\epsilon\geq 65$), the performance degradation compared to the non-private model progressively goes from moderate to negligible. Note that the maximum value of $\epsilon$ is around 80 in our experiments, which suggests it is possible to get reasonable privacy guarantees at a minimal cost in utility.

In Equation (\ref{eq:dp_guarantees}), the linear dependency of $\epsilon$ in the maximum number of predictions $T$ means that the privacy loss $\epsilon$ depends on the type of dataset used for predictions. In our case, Common Crawl data has long sentences, which means $\epsilon$ increases rapidly with $\lambda$. Another interesting aspect is the dependency on the vocabulary size $|V|$. In our approach, we sample from the entire vocabulary which translates in the privacy loss $\epsilon$ having a hard dependency on the number of tokens the model can predict. In practice, it is likely that only the top tokens should be considered, reducing the support of the distribution. A worthwhile future research direction is to improve the privacy guarantee of the proposed method by reducing the support size of the distribution. 

\section{Conclusion}
In this work, we study the problem of enhancing the privacy of an already trained LLM, without the necessity of retraining the model. We propose a simple perturbation mechanism to be applied to the discrete probability distributions output by the trained model, whose privacy benefits are demonstrated using the notion of differential privacy. Due to the post-processing property of differential privacy, the privacy improvements obtained in LLMs also transfer to the tasks for which such models are further fine-tuned. Note that while our mechanism perturbs the softmax layer output at the decoding stage, applying the same type of perturbation at the training stage has previously been studied under the name of ``label smoothing'' and shown improvements in terms of generalization and speed of learning in many domains \citep{muller2019does}. This suggests a future direction to investigate the generalization benefits of our proposed work. More broadly, this connection may support the observation that in some cases DP methods provide better generalization \citep{dwork2015generalization}.

There are also some natural extensions to the work introduced here. Further exploring structured perturbations as proposed in this paper, or considering random or adaptive perturbations in conjunction with other sophisticated sampling functions such as beam search are potentially useful next steps in this research direction. The privacy benefits obtained by such mechanisms can also be quantified by more specific metrics tailor-made to gauge model memorization; however, even defining such metrics is currently an active research endeavour \citep{Carlini2019TheSS, Carlini2022QuantifyingMA}.

\bibliography{anthology,custom}

\clearpage
\appendix

\section{Proof of Theorem \ref{thm:eps-dp}}
\label{ap:theory}
\begin{proof}
Note that Equation (\ref{eq:prob-pert}) can be rewritten as:
\begin{equation*}
\Pr[t_i=w_k \mid C_i] = q'_i(k)= \lambda q_i(k) + \dfrac{(1-\lambda)}{|V|}
\end{equation*}
which, combined with $0\leq q_i(k) \leq 1$, implies that 
\begin{equation}\label{eq:bounds}
\dfrac{1-\lambda}{|V|} \leq \Pr[t_i=w_k \mid C_i] \leq \lambda +\dfrac{1-\lambda}{|V|}.
\end{equation}

Then consider any two LMs $\mathcal{M}, \mathcal{M}'$, and an input sentence $x$ with $z \in \mathbb{N}$ masked tokens indexed by the set $[z]$. For any set of output tokens $y$, we have
\begin{equation}\label{eq:2}
\begin{aligned}
    &\dfrac{\Pr[\mathcal{M}(x) = y]}{\Pr[\mathcal{M}'(x) = y]} = \prod_{i \in [z]} \dfrac{\Pr[t_i=y_i \mid C_i]}{\Pr[t'_i=y_i \mid C'_i]}
\end{aligned}
\end{equation}
where $y_i$ denotes the token in $y$ at the $i^{th}$ masked position. Using Equation (\ref{eq:bounds}) in the above gives:
\begin{equation}\label{eq:3}
\begin{aligned}
    \prod_{i \in [z]} \dfrac{\Pr[t_i=y_i \mid C_i]}{\Pr[t'_i=y_i \mid C'_i]} &\leq \left( \dfrac{1+ (|V|-1)\lambda}{1- \lambda} \right)^z \\
    &\leq \left( \dfrac{1+ (|V|-1)\lambda}{1- \lambda} \right)^T
\end{aligned}
\end{equation}
where the last inequality above uses the fact that the maximum number of tokens predicted by $\mathcal{M}$ is $T$.

Combining Equations (\ref{eq:2}) and (\ref{eq:3}), and applying logarithmic transformation yields the desired result. 
\end{proof}

\section{Model Performance Examples}\label{ap:examples}
For each privacy level, we give three examples of MLM performance. The bold text in the true line indicates the tokens which were masked and the corresponding tokens in the predicted line indicate the ones predicted by the model. As observed in Figure \ref{fig:pu} at an aggregate level, the following examples show a deterioration in model performance as the privacy level improves. 

\begin{itemize}
    \item Privacy level: $\epsilon = \infty \;(\lambda = 1)$
    \begin{itemize}
        \item  \underline{\emph{True line}}: Our body does not naturally \textbf{produce} enough amino nutrients, so \\
        \underline{\emph{Predicted line}}: Our body does not naturally \textbf{produce} enough amino nutrients, so \\
        
        \item \underline{\emph{True line}}: \textbf{need} assistance after plung\textbf{ing} into the lagoon from the \textbf{slide} \\
        \underline{\emph{Predicted line}}: \textbf{medical} assistance after plung\textbf{ing} into the lagoon from the \textbf{sky} \\
        
        \item \underline{\emph{True line}}: We \textbf{came} \textbf{to} an area with temples. \\
        \underline{\emph{Predicted line}}: We \textbf{live} \textbf{in} an area with temples. 
    \end{itemize}
    \item Privacy level: $\epsilon \approx 63 \;(\lambda = 0.8)$
        \begin{itemize}
        \item  \underline{\emph{True line}}: it is \textbf{analog}, you have to play things all the \\
        \underline{\emph{Predicted line}}: it is \textbf{rewarding}, you have to play things all the \\
        
        \item \underline{\emph{True line}}: big fan \textbf{of} this poem for some unknown reason. (Report) \\
        \underline{\emph{Predicted line}}: big fan \textbf{of} this poem for some unknown reason. (Report) \\
        
        \item \underline{\emph{True line}}: \textbf{Biography} (among other categories), \textbf{it} has also been with Fiction. \\
        \underline{\emph{Predicted line}}: \textbf{Photography} (among other categories), \textbf{but} has also been with Fiction.
        \end{itemize}
    \item Privacy level: $\epsilon \approx 59 \;(\lambda = 0.6)$
    \begin{itemize}
        \item  \underline{\emph{True line}}: than most people do in their life \textbf{time}! And \textbf{it}s \\
        \underline{\emph{Predicted line}}: than most people do in their life \textbf{time}! And \textbf{here}s \\
        
        \item \underline{\emph{True line}}: I'\textbf{ve} been thrown by circumstances on a far \textbf{more} regular \\
        \underline{\emph{Predicted line}}: I'\textbf{ve} been thrown by circumstances on a far \textbf{more} regular \\
        
        \item \underline{\emph{True line}}: from \textbf{ever} happening \textbf{again}. Understand that your problem is simply  \\
        \underline{\emph{Predicted line}}: from \textbf{the} happening \textbf{again}. Understand that your problem is simply
    \end{itemize}
\end{itemize}

\end{document}